\pdfoutput=1
\newcommand{\shortversion}[1]{}
\newcommand{\longversion}[1]{#1}
 
\shortversion{
\relax
\documentclass[letterpaper]{article}
\usepackage{aaai21} 
\usepackage{times} 
\usepackage{helvet} 
\usepackage{courier} 
\usepackage[hyphens]{url} 
\usepackage{graphicx} 
\urlstyle{rm} 
\usepackage{graphicx} 
\usepackage{natbib}
\usepackage{caption}
\frenchspacing 
\setlength{\pdfpagewidth}{8.5in} 
\setlength{\pdfpageheight}{11in} 
\pdfinfo{
	/Title (Turbocharging Treewidth-Bounded Bayesian Network Structure Learning)
    /Author (Vaidyanathan Peruvemba Ramaswamy, Stefan Szeider)
	/TemplateVersion (2021.1)
    /Keywords (Fixed-Parameter Tractability, Constraint Satisfaction, Global Constraints, Satisfiability, Bayesian Networks, Normal Logic Programs, Computational Complexity)
}

\setcounter{secnumdepth}{2} 
}

\longversion{\documentclass[letterpaper]{article}
\usepackage[totalwidth=420pt,totalheight=625pt]{geometry}
\date{}
\usepackage{times}
\usepackage[numbers,compress]{natbib}
\usepackage{hyperref}       
\usepackage{url}            
\urlstyle{rm}
\usepackage{graphicx}
\usepackage{caption}
}

\makeatletter
\longversion{

}
\makeatother

\usepackage[utf8]{inputenc} 
\usepackage[T1]{fontenc}    
\usepackage{booktabs}       
\usepackage{nicefrac}       
\usepackage{microtype}      
 
\shortversion{
}
\usepackage{amsmath}
\usepackage{amssymb}
\usepackage{amsfonts}       
\usepackage{amsthm}
\newtheorem{theorem}{Theorem}
\newtheorem{claim}{Claim}

\usepackage{enumitem}

\usepackage{xcolor}
\colorlet{MyBlue}{blue!50!black!100!}
\colorlet{MyRed}{red!50!black!100!}
\colorlet{MyGreen}{green!50!black!100!}

\newcommand{\note}[1]{}
\newcommand{\todo}[1]{}

\newcommand{\TTT}{\mathcal{T}}
\newcommand{\SSS}{\mathcal{S}}
\newcommand{\PPP}{\mathcal{P}}

\newcommand{\SB}{\{\,}%
\newcommand{\SM}{\;{:}\;}%
\newcommand{\SE}{\,\}}%

\newcommand{\Card}[1]{|#1|}
\let\phi=\varphi
\let\epsilon=\varepsilon

\newcommand{\ntxt}[1]{\text{{\normalfont #1}}}

\newcommand{\new}{\ntxt{new}}
\newcommand{\virt}{\ntxt{virt}}
\newcommand{\ext}{\ntxt{ext}}

\newcommand{\bnslim}{BN\nobreakdash-SLIM}
\newcommand{\kmax}{k\nobreakdash-MAX}
\newcommand{\kgreedy}{k\nobreakdash-greedy}
\newcommand{\kslim}{\bnslim(\kmax)}
\newcommand{\hc}{ETL\textsubscript{d}}   
\newcommand{\hcp}{ETL\textsubscript{p}}  
\newcommand{\hcb}{ETL}                   

\newcommand{\hcbslim}{\bnslim(\hcb)}

\longversion{
\newcommand{\mysubsection}[1]{\subsection{#1}}
}
\shortversion{
\newcommand{\mysubsection}[1]{\medskip\noindent\textbf{#1}\quad}
}

\shortversion{
\title{Turbocharging Treewidth-Bounded\\ Bayesian Network Structure
  Learning}
\author{%
Vaidyanathan Peruvemba Ramaswamy and  Stefan Szeider\\}
\affiliations {%
Algorithms and Complexity Group, TU Wien, Vienna, Austria\\
}
}
\longversion{
\title{Turbocharging Treewidth-Bounded\\ Bayesian Network Structure
	Learning\thanks{This is the full version of a paper to appear in
        the proceedings of AAAI-21, the Thirty-Fifth AAAI Conference on
        Artificial Intelligence. The authors acknowledge the support by the FWF
          (projects P32441 and W1255) and by the WWTF (project
          ICT19-065). }}
\author{%
Vaidyanathan P. R. and Stefan Szeider\\[4pt]
\small Algorithms and Complexity Group\\[-3pt]
\small TU Wien, Vienna, Austria\\[-3pt] 
\small \texttt{\{vaidyanathan,sz\}@ac.tuwien.ac.at}
}
}

\begin{document}
\thispagestyle{empty}

\maketitle

\begin{abstract}
We present a new approach for learning the structure of a
treewidth-bounded Bayesian Network~(BN). The key to our approach is
applying an exact method~(based on MaxSAT) locally, to improve the
score of a heuristically computed BN. This approach allows us to scale
the power of exact methods---so far only applicable to BNs with
several dozens of random variables---to large BNs with several thousands of
random variables.
Our experiments show that our method improves the score of
  BNs provided by state-of-the-art heuristic methods, often significantly.
\end{abstract}

\section{Introduction}
Bayesian network structure learning is the notoriously difficult
problem of discovering a Bayesian network (BN) that optimally
represents a given set of training data \cite{Chickering96}. Since
exact inference on a BN is exponential in the  BN's treewidth
\cite{KwisthoutBodlaenderGaag10}, one is particularly interested in
learning BNs of bounded treewidth.  However, learning a BN of bounded
treewidth that optimally fits the data (i.e., with the largest possible
score) is, in turn, an NP-hard task~\cite{KorhonenParviainen13}. This
predicament caused the research on treewidth-bounded BN structure
learning to split into two branches:
\begin{enumerate}[leftmargin=*,itemsep=0cm]
\item \emph{Heuristic Learning} (see, e.g.,
  \citet{ElidanGould08,NieCamposJi15,ScanagattaCoraniCamposZaffalon16,ScanagattaCZYK18,BenjumedaCamposLarranaga19}),
  which is scalable to large BNs with thousands of random variables but with a
  score that can be far from optimal, and
\item \emph{Exact Learning} (see, e.g.,
  \citet{BergJarvisaloMalone14,KorhonenParviainen13,ParviainenFarahaniLagergren14}),
  which learns optimal BNs but is scalable only to a few dozen
  random variables.
\end{enumerate}
In this paper, we combine heuristic and exact learning and take the
best of both worlds.

The basic idea for our approach is to first compute a BN  with
a  heuristic method (the \emph{global solver}), and then to
apply an exact method (the \emph{local solver}) to parts of the
heuristic solution. The parts are chosen small enough that they allow an optimal
solution reasonably quickly with the exact method. Although the basic idea sounds
compelling and reasonably simple, its realization requires several
conceptual contributions and new results.

For the global solver, any heuristic algorithm for treewidth-bounded
BN learning, such as the recent algorithms \kmax~\cite{ScanagattaCZYK18}
or \hcb~\cite{BenjumedaCamposLarranaga19}.
The local solver's task is significantly more complex than
treewidth-bounded BN structure learning, as several additional
constraints need to be incorporated. 
Namely, it is not sufficient that the BN computed by the local solver is
acyclic.  We need \emph{fortified acyclicity constraints} that prevent
cycles that run through the other parts of the BN, which have not been
changed by the local solver. Similarly, it is not sufficient that the
local BN is of bounded treewidth.  We need \emph{fortified treewidth
  constraints} that prevent the local BN from introducing links
between a diverse set of nodes that, together with the other parts of
the BN, which have not been changed by the local solver, increase the
treewidth.

Given these additional requirements, we propose a new local solver
\bnslim\ (\emph{SAT-based Local Improvement Method}), which satisfies the
fortified constraints.  We formulate a fortified version of the
treewidth-bounded BN structure learning problem. In
Theorem~\ref{the:slim}, we show that we can express the fortified
constraints with  certain \emph{virtual arcs} and
\emph{virtual edges}.  The virtual arcs represent directed paths that
run outside the local instance; with these virtual arcs we can ensure
fortified acyclicity. The \emph{virtual edges} represent essential
parts of a global tree decomposition using which we can ensure bounded
treewidth.

The new formulation of the local problem is well-suited to be
expressed as a MaxSAT (\emph{Maximum Satisfiability}) problem and
hence allows us to harvest the power of state-of-the-art MaxSAT
solvers (which received a significant performance gain over the last
decade).  A distinctive feature of our encoding is that, in contrast
to the virtual edges, the virtual arcs are conditional and depend on
the local solver's solution.
 
\mysubsection{Results}
We implement \bnslim\ and evaluate it empirically on a large set of
benchmark data sets, consisting between 64 and 10,000 random variables 
and for the treewidth bounds 2, 5, and 8. As the global solver, we use the
state-of-the-art heuristic algorithms for treewidth-bounded BN
learning \kmax~\cite{ScanagattaCZYK18}, and two variants of~\hcb\ \cite{BenjumedaCamposLarranaga19}.  \kmax\ improves over the
\kgreedy\ algorithm \cite{ScanagattaCoraniCamposZaffalon16}, which was
the first algorithm for treewidth-bounded structure learning that
scaled to thousands of random variables. The more recent algorithm \hcb\
is reported to perform better than \kmax\ in many
cases~\cite{BenjumedaCamposLarranaga19}.  

We consider about a hundred benchmark data sets based on real-world
and synthetic data sets, ranging up to 4000 random variables in our
experiments.  First we run the global solvers on the data sets,
followed by running \bnslim\ to improve the score of the DAG they
provided. Our results show that after running \bnslim\ for 5 minutes,
73\%
of all DAGs could be improved; 
by extending the time for \bnslim\ to 15 minutes, the improvement extends
to 82\%.
We also notice that, overall, \bnslim\ can improve the lower treewidth
DAGs more efficiently.

Since \kmax\ is an \emph{anytime} algorithm that can produce better
and better solutions over time, we can directly compare the
improvements achieved by \kmax\ after some initial run with the
improvements achieved by \bnslim. Our experiments show that after an
initial run of \kmax\ for 30 minutes, it is highly beneficial to stop
\kmax\ and hand the torch over to \bnslim, as \bnslim\ provides
improvements at a significantly higher rate. According to the
$\Delta$BIC metric, which was used by~\citet{ScanagattaCZYK18} for
comparing treewidth-bounded BN structure learning algorithms, the
results are ``extremely positive'' in favor of \bnslim\ over \kmax\ in
a vast majority of the experiments.

We cannot perform such a direct comparison between \hcb\ and
\bnslim, since the available implementation of \hcb\ does not
support an anytime run, but stops after a certain time.  Hence, we let
\hcb\ finish, and run \bnslim\ afterwards for 30 minutes. 
The achieved improvement in terms of the $\Delta$BIC metric is
``extremely positive'' for 84\% of all DAGs computed by two
  variants of \hcb. 

\mysubsection{Related work}
The first SAT-encoding for finding the treewidth of a graph was proposed
by~\citet{SamerVeith09}.  \citet{FichteLodhaSzeider17} proposed the first SAT-based local
improvement method for treewidth, using the Samer-Veith encoding as
the local solver. Recently, SAT encodings have been proposed for other
graph and hypergraph width measures \cite{%
  FichteHecherLodhaSzeider18, GanianLodhaOrdyniakSzeider19,
  LodhaOrdyniakSzeider17, SchidlerSzeider20}. 
So far, there have been four concrete approaches that use the SLIM framework,
one for branchwidth~\cite{LodhaOrdyniakSzeider16,LodhaOrdyniakSzeider19},
one for treewidth~\cite{FichteLodhaSzeider17}, one
for treedepth~\cite{VaidyanathanSzeider2020} and 
one for decision~trees~\cite{SchidlerSzeider21}.

Several exact approaches to treewidth-bounded BN structure learning
have been proposed.
\citet{KorhonenParviainen13} proposed a dynamic-programming approach,
and \citet{ParviainenFarahaniLagergren14} proposed a
Mixed-Integer Programming approach. 
\citet{BergJarvisaloMalone14} proposed a MaxSAT approach by
extending the basic Samer-Veith encoding for treewidth. Our approach
for \bnslim\ uses a similar general strategy, but we encode acyclicity 
differently. Moreover, \bnslim\  deals with the
fortified constraints in terms of virtual edges and virtual arcs.

Since the exact methods are limited to small domains, 
\citet{NieCamposJi15, NieCamposJi16} suggested heuristic approaches
that scale up to hundreds of random variables.  The \kgreedy\ algorithm
proposed by \citet{ScanagattaCoraniCamposZaffalon16}
at NIPS'16 provided a breakthrough, consistently yielding better DAGs
than its competitors and scaling up to several thousand of random variables.
As mentioned above, \kmax~\cite{ScanagattaCZYK18} is a more recent
improvement over \kgreedy.
More recently, \citet{BenjumedaCamposLarranaga19} came up with 
the~\hcb\ algorithms, based on local search
within the space of structures called elimination trees.
These algorithms perform better than k\nobreakdash-MAX and \kgreedy\ in many cases.

\section{Preliminaries}

\mysubsection{Structure learning}
We consider the problem of learning the structure (i.e., the~DAG) of a
BN from complete data set of~$N$ instances $D_1,\dots,D_N$ over a set
of~$n$ categorical random variables $X_1,\dots,X_n$. The goal is to
find a DAG~$D=(V,E)$ where~$V$ is the set of nodes (one for each
random variable) and~$E$ is the set of arcs (directed edges).
The value of a \emph{score function} determines how well a DAG~$D$ fits
the data; the DAG~$D$, together with local parameters, forms the BN \cite{KollerFriedman09}.

We assume that the score is \emph{decomposable}, i.e., being
constituted by the sum of the individual random variables' scores.
Hence we can assume that the score is given in terms of a \emph{score
  function $f$} that assigns each node $v\in V$ and each subset
$P \subseteq V\setminus \{v\}$ a real number $f_P(v)$, the
\emph{score} of~$P$ for~$v$.  The score of the entire DAG~$D=(V,E)$ is
then
\longversion{
\[ f(D) := \sum_{v\in V} f(v,P_D(v)) \]
}
\shortversion{
\( f(D) := \sum_{v\in V} f(v,P_D(v)) \) 
}
where $P_D(v)=\SB u\in V \SM (u,v)\in E \SE$ denotes the \emph{parent
  set} of~$v$ in~$D$. This setting accommodates several popular scores
like AIC, BDeu, and
BIC~\cite{Akaike74,HeckermanGeigerChickering95,Schwarz78}. If $P$ and
$P'$ are two potential parent sets of a random variable~$v$ such
that~$P \subsetneq P'$ and~$f(v, P') \leq f(v, P)$, then we can safely
disregard the potential parent set~$P'$ of~$v$.
Consequently, we can disregard all nonempty potential parent sets
of~$v$ with a score~$\leq f(v, \emptyset)$.  Such a restricted score
function is a \emph{score function cache}.
 
\mysubsection{Treewidth}
Treewidth is a graph invariant that provides a good indication of how
costly probabilistic inference on a BN~is. Treewidth is defined on
undirected graphs and applies to BNs via the \emph{moralized graph}
$M(D)=(V,E_M)$ of the DAG $D=(V,E)$ underlying the BN under
consideration, where
\(E_M=\SB \{ u, v \} \SM (u,v)\in E \SE \cup \SB \{u, v\} \SM
(u,w),(v,w)\in~E, u\neq~v \SE.\)

A \emph{tree decomposition}~$\TTT$ of a graph $G$ is a pair
$(T,\chi)$, where $T$ is a tree and $\chi$ is a function that assigns
each tree node $t$ a set $\chi(t)$ of vertices of~$G$ such that the
following conditions hold:
\begin{description}
\item[T1] For every edge $e$ of~$G$ there is a tree node $t$ such
  that $e\subseteq \chi(t)$.
\item[T2] For every vertex $v$ of~$G$, the set of tree nodes $t$
  with $v\in \chi(t)$ induces a non-empty subtree of~$T$.
\end{description}
The sets $\chi(t)$ are called \emph{bags} of the
decomposition~$\TTT$, and $\chi(t)$ is the bag associated with
the tree node~$t$. The \emph{width} of a tree decomposition $(T,\chi)$
is the size of a largest bag minus~$1$. The \emph{treewidth} of~$G$,
denoted by $\textup{tw}(G)$, is the minimum width over all tree
decompositions of~$G$.

The \emph{treewidth-bounded BN structure learning problem}
takes as input a set $V$ of nodes, a decomposable score
function $f$ on $V$, and an integer $W$, and it asks to compute a DAG
$D=(V,E)$ of treewidth $\leq W$, such that $f(D)$ is maximal.

\section{Local improvement}







Consider an instance $(V,f,W)$ of the treewidth-bounded BN structure learning
problem, and assume we
have computed  an \emph{initial solution} $D=(V,E)$ heuristically, together with a tree
decomposition $\TTT=(T,\chi)$ of width $\leq W$ of the moralized graph
$M(D)$.

We select a subtree $S\subseteq T$ such that the number of vertices in
$V_S:=\bigcup_{t\in V(S)} \chi(t)$ is at most some \emph{budget}~$B$.
The budget is a parameter that we specify beforehand, such that the
subinstance induced by $V_S$ is small enough to be solved
optimally by an exact method, which we call the \emph{local solver}.
The local solver computes for each $v\in V_S$ a new parent set,
optimizing the score of the resulting DAG $D^\new=(V,E^\new)$.

Consider the induced DAG $D_S^\new=(V_S,E_S^\new)$, where
$E_S^\new=\SB (u,v)\in E^\new\SM \{u,v\}\subseteq  V_S\SE$. The local solver ensures
that the following conditions are met:
\begin{enumerate}[align=left,font=\bfseries]
\item[C1] $D_S^\new$ is acyclic.
\item[C2] The moral graph $M(D_S^\new)$ has treewidth $\leq W$.
\end{enumerate}

\longversion{\pagebreak[2]}

We assume that the local solver certifies C2 by producing a tree
decomposition $\SSS^\new~=~(S^\new,\chi^\new)$ of~$M(D_S^\new)$ of
width~$\leq W$, which can be used by the global solver.

The two conditions stated above are not sufficient to ensure that
$D^\new$ is acyclic and that treewidth of~$M(D^\new)$ remains bounded
by $W$. Acyclicity can be violated by cycles formed by the combination
of the new incoming arcs of vertices in $S$ together with old arcs
that are kept from $D$. The treewidth can increase by a number
that is linear in $\Card{V_S}$.
 
Hence, we need additional side conditions, which we will formulate
using the following  additional concepts.

Let us call a vertex $v\in V_S$ a \emph{boundary vertex} if there
exists a tree node $t\in V(T)\setminus V(S)$ such that $v\in \chi(t)$,
i.e., it occurs in some bag outside~$S$. We call the other vertices in $V_S$
\emph{internal vertices}, and the vertices in $V\setminus V_S$
\emph{external vertices}.  Further, we call two boundary
vertices $v,v'$ \emph{adjacent} if there exists a tree node
$t\in V(T)\setminus V(S)$ such that $v,v'\in \chi(t)$, i.e., both
vertices occur together in some bag outside~$S$. It is easy to see
that any pair of adjacent boundary vertices occur together in a bag of
$S$ as well.

For any two adjacent boundary vertices $v,v'$, we call $\{v,v'\}$ a
\emph{virtual edge}. 
Let $E_\virt$ be the set of all virtual edges.
These virtual edges form a clique and serve a similar purpose
as the marker cliques used in other work~\cite{FichteLodhaSzeider17}.
The \emph{extended moral graph} $M_\ext=(V_S,E_\ext)$ is obtained
from $M(D_S^\new)$ by adding all virtual edges.

For any two adjacent boundary vertices $v,v'$, we call $(v',v)$ a
\emph{virtual arc}, if $D^\new$ contains a directed path from $v'$ to
$v$, where all the vertices on the path, except for $v'$ and $v$, are
external. Let $E_\virt^\rightarrow$ be the set of all virtual arcs.
 

We can now formulate the side conditions.
\begin{enumerate}[align=left,font=\bfseries]
\item[C3] $\SSS^\new$ is a tree decomposition of the extended moral graph~$M_\ext$.
\item[C4] For each $v\in V_S$, if $P_{D^\new}(v)$ contains external
  vertices, then there is some $t\in V(T)\setminus V(S)$ such that
  $P_{D^\new}(v) \cup \{v\} \subseteq \chi(t)$.
\item[C5] The digraph $(V_S,E_S^\new \cup E_\virt^\rightarrow)$ is acyclic.
\end{enumerate}
We note that condition C4 implies that in $D^\new$, all parents of an
internal vertex are internal.

 
\begin{figure}[htbp]
	\centering
	\includegraphics[width=.6\linewidth]{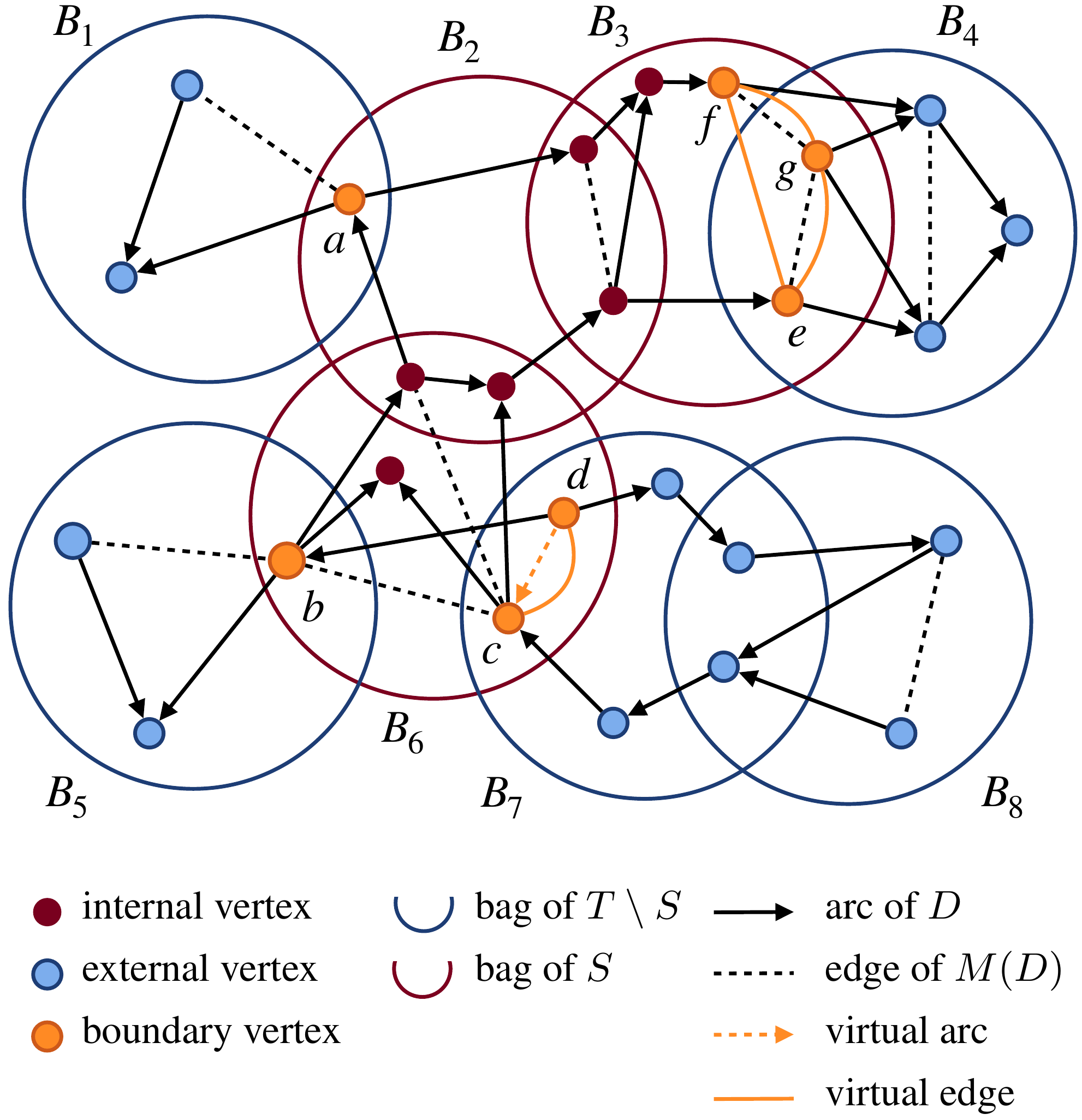}
	\caption{Illustration for Theorem~\ref{the:slim}. The large
          circles $B_1,\dots,B_8$ represent the bags of~$T$, where $B_2,B_3,B_6$
          belong to $S$. The boundary vertices are $a,\dots,g$, where
          $c,d$ are adjacent, and $e,f,g$ are mutually adjacent. Since
          there is a directed path from $d$ to $c$ using external
          vertices from the bags $B_7$ and $B_8$, there is a virtual arc
          from $d$ to $c$.}
\end{figure} 

\shortversion{\vspace{-10pt}}
\begin{theorem}\label{the:slim}
  If all the conditions C1--C5 are satisfied, then $D^\new$ is acyclic,
  the treewidth of~$M(D^\new)$ is at most $W$, and the score of~$D^\new$ is
  at least the score of~$D$. 
\end{theorem}
\begin{proof}
  We define a new tree decomposition $\TTT^\new=(T^\new,\chi^\new)$ of~$M(D^\new)$
   as follows. Let $T_1,\dots,T_r$ be the connected
  components of~$T\setminus V(S)$, i.e., the $T_i$'s are the subtrees
  of~$T$ that we get when deleting the subtree $S$. Let
  $V_i=\bigcup_{t\in V(T_i)}\chi(t)$, $1\leq i \leq r$, and observe that
  each external vertex $x$ belongs to exactly one of the sets
  $V_1,\dots,V_r$. Let $B_i=V_S\cap V_i$, $1\leq i \leq r$, be the set of
  boundary vertices in $V_i$. We observe that all the vertices in
  $B_i$ are mutually adjacent boundary vertices and occur together in
  a bag~$\chi(s_i)$ of~$s_i\in V(S)$ and in a bag $\chi(t_i)$, for
  $t_i\in V(T_i)$, as we can take $s_i$ and $t_i$ to be the two neighboring
  tree nodes of~$T$ with $s_i\in V(S)$ and $t_i\in V(T_i)$. We also
  observe that each $B_i$ forms a clique in the extended moral graph
  $M_\ext$.

  Recall that by assumption, the local solver provides a tree
  decomposition $\SSS^\new=(S^\new,\chi^\new)$ of~$D_S^\new=(V_S,E_S^\new)$ of width
  $\leq W$.  Additionally, by condition C3, $\SSS^\new$ is also a tree
  decomposition of~$M_\ext$, and hence, by a basic property of tree
  decompositions (see, e.g., \citet[Lem.~3.1]{BodlaenderMoehring93}),
  there must exist a bag $\chi^\new(s_i^*)$, $s_i^*\in V(S^\new)$, with
  $B_i\subseteq \chi^\new(s_i^*)$.  Hence we can define~$T^\new$ as the tree
  we get by connecting the disjoint trees $S^\new,T_1,\dots,T_r$ with the
  edges $\{s_i^*,t_i\}$, $1\leq i \leq r$. We extend $\chi^\new$ from
  $V(S^\new)$ to $V(T^\new)$ by setting $\chi^\new(t)=\chi(t)$ for
  $t\in \bigcup_{i=1}^r V(T_i)$.

\longversion{
  \begin{claim}\label{claim:width}
  	$\TTT^\new=(T^\new,\chi^\new)$ is a tree decomposition of~$M(D^\new)$ of
  	width $\leq W$.
  \end{claim}
}
\shortversion{
  \emph{Claim 1:
     $\TTT^\new=(T^\new,\chi^\new)$ is a tree decomposition of~$M(D^\new)$ of
    width $\leq W$.}
}
    To prove the claim, we show that $\TTT^\new$ satisfies the conditions
  T1 and T2.
  
  \emph{Condition T1}. There are two reasons for an
  edge $\{u,v\}$ to belong to $M(D^\new)$: first, because of an arc
  $(u,v)\in E^\new$ and second, because of two arcs $(u,w),(v,w)\in
  E^\new$. 
  \emph{First case: $(u,v)\in E^\new$}. If $u$ and $v$ are both
  external, then $\{u,v\}\subseteq \chi(t)=\chi^\new(t)$ for some
  $t\in V(T^\new)\setminus V(S^\new)= V(T)\setminus V(S)$. If neither
  $u$ nor $v$ is external, then $(u,v)\in E_S^\new$, and since
  $S^\new$ is a tree decomposition of~$D_S^\new$,
  $\{u,v\}\subseteq \chi^\new(s)$ for some $s\in V(S^\new)$.  If $v$
  is external but $u$ isn't, then the arc $(u,v)$ was already present
  in $D$, as the parents of external vertices didn't change. Hence,
  since $\TTT$ is a tree decomposition of~$M(D)$, it follows that
  $\{u,v\}\subseteq \chi(t)=\chi^\new(t)$ for some
  $t\in V(T^\new)\setminus V(S^\new)= V(T^\new)\setminus V(S)$.  If $u$
  is external but $v$ isn't, it follows from C4 that
  $\{u,v\}\subseteq \chi(t)=\chi^\new(t)$ for some
  $t\in V(T^\new)\setminus V(S^\new)= V(T^\new)\setminus V(S)$.
  \emph{Second case: $(u,w),(v,w)\in E^\new$}.  If $u,v,w \in V_S$,
  then $\{u,v\}\in E(M(D_S^\new))$, and so $\{u,v\}\subseteq \chi^\new(s)$ for
  some $s\in V(S^\new)$, since $\SSS^\new$ is a tree decomposition of
  $M(D_S^\new)$.  If $w\in V_S$ but $u \notin V_S$ or $v \notin V_S$, then
  C4 implies that $\{u,v\}\subseteq \chi(t)=\chi^\new(t)$ for some
  $t\in V(T^\new)\setminus V(S^\new)= V(T^\new)\setminus V(S)$.  If
  $w\notin V_S$, then $u,v$ are two adjacent boundary vertices, hence
  $\{u,v\}$ is a virtual edge which, by C3, means
  $\{u,v\}\subseteq \chi^\new(s)$ for some $s\in V(S^\new)$.  We conclude
  that T1~holds.

  \emph{Condition T2}. Let $v \in V$. If $v$ is external, then there
  is exactly one $i\in \{1,\dots, r\}$, such that
  $v\in V_i=\bigcup_{t\in V(T_i)}\chi(t)$.  Since we do not change the
  tree decomposition of~$T_i$, condition T2 carries over from $\TTT$
  to~$\TTT^\new$.  Similarly, if $v$ is internal, then $v$ does not
  appear in any bag $\chi^\new(t)$ for
  $t\in V(T^\new)\setminus V(S^\new)$, hence condition T2 carries over
  from $\SSS^\new$ to $\TTT^\new$. It remains to consider the case
  where $v$ is a boundary vertex.  The tree nodes $t\in~V(S^\new)$ with
  $v\in \chi(t)$ are connected, because $\SSS^\new$ satisfies T2, and
  for~$B_i: v\in B_i$, the tree nodes $t\in V(T_i)$ for which
  $v\in \chi(t)$ are connected, since $\TTT$ satisfies T2.  By
  construction of~$T^\new$, if $v\in B_i$, then there are neighboring
  tree nodes $s_i^*\in V(S^\new)$ and $t_i \in V(T_i)$ with
  $v\in \chi^\new(s_i^*)\cap \chi^\new(t_i)$. Hence all the tree nodes
  $t\in V(T^\new)$ with $v\in \chi(t)$ are connected, and T2 also holds
  for boundary vertices.

  To conclude the proof of the claim, 
  it remains to observe the width of~$\TTT^\new$ cannot exceed the
  widths of~$\TTT$ or $\SSS^\new$, hence the width of~$\TTT^\new$ is
  at most~$W$.

\longversion{
 \begin{claim}\label{claim:acyc}
   $D^\new$ is acyclic.
\end{claim}
}
\shortversion{
 \emph{Claim 2:
    $D^\new$ is acyclic.}
}
  To prove the claim, suppose to the contrary that $D$ contains a
  directed cycle $C=(V(C),E(C))$. The cycle cannot lie entirely in
  $D_S^\new$, nor can it lie entirely in $D^\new-V_S=D-V_S$, because
  $D_S^\new$ and $D$ are acyclic. Hence, $C$ contains at least one arc
  from $V_S \times (V\setminus V_S)$ and at least one arc from
  $ (V\setminus V_S) \times V_S$.  Let
  $(v_j,x_j) \in E(C) \cap (V_S \times (V\setminus~V_S))$ and
  $(x_j',v_j') \in E(C) \cap ((V\setminus V_S) \times V_S)$, for
  $0\leq j \leq p$, be these arcs, such that they appear on $C$ in the
  order $(v_0',x_0')$, $(x_0,v_0)$, $(v_1',x_1'),\dots,(v_p',x_p')$,
  $(x_p,v_p)$. It is possible that $x_j'=x_j$ or $v_j = v_{j+1}'$.  We
  observe that the vertices on the path from $x_j'$ to~$x_j$ on $C$
  all belong to some $V_i = \bigcup_{t\in V(T_i)}\chi(t)$. Hence $v_j$
  and $v_j'$ are adjacent boundary vertices, and $E_\virt^\rightarrow$
  contains all the arcs $(v_j',v_j)$, $1\leq j \leq p$. However, the
  cycle $C$ contains also the paths from $v_j$ to $v'_{j+1 \pmod p}$,
  for~$1\leq j \leq p$, which only run through vertices in
  $V_S$. These paths, together with the virtual arcs $(v_j',v_j)$ form
  a cycle $C'$ which lies in $(V_S,E_S^\new \cup
  E_\virt^\rightarrow)$. This contradicts C5 which requires that this
  digraph be acyclic. Hence the claim holds. 

\longversion{
	\begin{claim}
		The score of~$D^\new$ is at least the score of~$D$.
	\end{claim}
}
\shortversion{
 \emph{Claim 3:
    The score of~$D^\new$ is at least the score of~$D$.}
}
 We observe that by taking $D^\new = D$ we have a solution that
 satisfies all the required conditions and maintains the score.
 \shortversion{\looseness=-1}  
\end{proof}


\section{Implementing the local improvement}


\newcommand{\Avirt}[2]{A_\virt^\rightarrow(#1, #2)}

In this section, we first discuss how the set $S$ representing the
subinstance is constructed. Then we provide a detailed explanation of
the MaxSAT encoding that is responsible for solving the subinstance.

\mysubsection{Constructing the subinstance}
For this section, we follow the same notation as used in the previous section.
To construct the subinstance, we initialize the subtree $S$ with a tree node $r$ 
picked at random from $V(T)$. We then expand $S$ by performing a bread-first search
from $V(S)$ and adding a new tree node to $S$ as long as the size of~$V_S$ does not exceed
the budget. Next, we compute $E_\virt$ for the chosen $S$.
Finally, we prune the parent sets of each vertex so as to only retain those 
parent sets which satisfy conditions C3 and C4. This can be done by 
first checking, for each parent set, if the required tree node $t$ is present 
$V(T)\setminus V(S)$, and if it does, we record the set of virtual arcs 
that are imposed by this parent set as long as none of the virtual arcs
are self-loops. For each~$v \in S$ and $P \in \PPP_v$, we denote by $\Avirt vP$
the set of imposed virtual arcs when $v$ has the parent set~$P$ in $D^\new$.
We denote by $\PPP_v$, the collection of parent sets of 
node $v$ that remain after this pruning process. Notice that, under  this pruning,
all remaining parent sets $P \in \PPP_v$ satisfy C4. 
Also note that, since~$E_\virt^\rightarrow$ is conditional on the chosen parent
sets, it cannot be precomputed.

Further, since we intend to solve the subinstance using a MaxSAT encoding,
we need to ensure that the score of each parent set is non-negative. 
Recall that $\PPP_v$ only contains those non-empty parent sets
whose score is at least that of the empty parent set. Thus, we may assume that
the empty parent set has the lowest score among all the parents of a certain
vertex. Consequently, we can adjust the score function  
by setting $f'_P(v) = f_P(v) - f_\emptyset(v)$ for $v \in S$ and $P \in \PPP_v$,
which implies that $f'_P(v) \geq 0$ for all $v \in S$ and $P \in \PPP_v$.

\newcommand{\ps}[2]{\ntxt{par}_#1^#2}  
\newcommand{\acyc}[2]{\ntxt{acyc}_{#1,#2}}
\newcommand{\acycs}[2]{\ntxt{acyc}^*_{#1,#2}}
\newcommand{\ord}[2]{\ntxt{ord}_{#1,#2}}
\newcommand{\ords}[2]{\ntxt{ord}^*_{#1,#2}}
\newcommand{\arc}[2]{\ntxt{arc}_{#1,#2}}

\mysubsection{MaxSAT encoding}
We now describe the weighted partial MaxSAT instance that encodes 
conditions C1--C5.
We build on top of the SAT encoding proposed by \citet{SamerVeith09}.
The only difference in our case is that there are no explicit edges and hence we do
not require the corresponding clauses. Instead, the edges of the moralized graph
are dependent on and decided by other variables that govern the DAG structure.
For convenience, let $n$ denote the size of the subinstance, i.e., $n := \Card{S}$.
A part of the encoding is based on the elimination ordering of a tree 
decomposition (see, e.g.,~\citet[Sec.~2]{SamerVeith09}).

The main variables used in our encoding are
\begin{itemize}[leftmargin=*,nosep]
  \item variables $\ps vP$ represent for each node $v \in S$
        the chosen parent set $P$,
  \item $n(n-1)/2$ variables $\acyc uv$ represent the topological ordering
        of~$D_S^\new$,
  \item $n(n-1)/2$ variables $\ord uv$ represent the elimination ordering
        of the tree decomposition,
  \item $n^2$ variables $\arc uv$ represent the arcs in the moralized graph $M_\ext$,
        along with the \emph{fill-in edges}~(see~\citet{SamerVeith09}).
        \note{add prelims about elim order to explain fill-in edges?}
\end{itemize}

\shortversion{  
\setlength{\abovedisplayskip}{3pt}
\setlength{\belowdisplayskip}{4pt minus 1pt}
\setlength{\abovedisplayshortskip}{0pt}
\setlength{\belowdisplayshortskip}{1pt plus 1pt minus 1pt}
}
Since $\acyc uv$ and $\ord uv$ represent linear orderings, we enforce 
transitivity of these variables by means of the clauses 
\begin{equation*}\label{eqn:trans}
\left.\begin{aligned}
  (\acycs uv \wedge \acycs vw) &\rightarrow \acycs uw \\
  (\ords uv \wedge \ords vw) &\rightarrow \ords uw
\end{aligned}\right\}
  \text{ for distinct } u, v, w \in S.
\end{equation*}
To prevent self-loops in the moralized graph, we add the clauses
\begin{equation*}\label{eq:no-self-loops}
  \neg \arc vv \quad \text{ for } v \in S.
\end{equation*}
 
For each node $v \in S$, and parent set $P \in \PPP_v$, the variable $\ps vP$ is true 
if and only if $P$ is the parent set of~$v$. Since each node must have 
exactly one parent set, we introduce the cardinality constraint
\longversion{
\begin{equation*}
  \sum_{P \in \PPP_v} \ps vP = 1 \text{ for } v \in S.
\end{equation*}
}
\shortversion{
  $\sum_{P \in \PPP_v} \ps vP = 1$  for $v \in S$.
}

Next, for each node $v$, parent set $P$, and $u \in P$, if $P$ is 
the parent set of~$v$ then $u$ must precede $v$ in the topological ordering.
Hence we add the clause
\begin{equation*}\label{eq:par->acyc}
  \ps vP \rightarrow \acyc uv \quad \text{for } v \in S, P \in \PPP_v,
      \text{ and } u \in P.
\end{equation*}
Similarly, for each node $v$, parent set $P$, and $u \in P$, if $P$ is
the parent set of~$v$ then we must add an arc in the moralized graph 
respecting the elimination ordering between $u$ and $v$, as follows:
\begin{equation*}\label{eq:par^ord->arc}
\left.\begin{aligned}
  (\ps vP \wedge \ord uv) &\rightarrow \arc uv \\
  (\ps vP \wedge \ord vu) &\rightarrow \arc vu
  \end{aligned}\right\}
  \begin{aligned}\quad
      \text{for } v \in S, P \in \PPP_v,\\
      \text{ and } u \in P.
  \end{aligned}
\end{equation*}

Next, we encode the moralization by adding an arc between every pair of 
parents of a node, using the following clauses
\begin{equation*}\label{eq:par^ord->moral}
\left.\begin{aligned}
  (\ps vP \wedge \ord uw) &\rightarrow \arc uw \\
  (\ps vP \wedge \ord wu) &\rightarrow \arc wu
  \end{aligned}\right\}
  \begin{aligned}\quad
      \text{for } v \in S, P \in \PPP_v,\\
      \text{ and } u,w \in P.
  \end{aligned}
\end{equation*}
Now, we encode the fill-in edges, with the following clauses
\begin{equation*}\label{eq:fillin}
\left.\begin{aligned}
  (\arc uv \wedge \arc uw \wedge \ord vw) \rightarrow \arc vw \\
  (\arc uv \wedge \arc uw \wedge \ord wv) \rightarrow \arc wv
  \end{aligned}\right\}
      \quad\text{for } u,v,w \in S.
\end{equation*}
Lastly, to bound the treewidth, we add a cardinality constraint
on the number of outgoing arcs for each node as follows
\begin{equation*}\label{eq:width}
  \sum\nolimits_{w \in S, w \neq v} \arc vw \leq W \quad \text{for } v \in S.
\end{equation*}
To complete the basic encoding, for every node $v \in S$, 
and every parent set $P \in \PPP_v$ we add a soft clause weighted by 
the score of the parent set as follows
\begin{equation*}\label{eq:soft-clause}
  (\ps vP):\text{weight } f'_P(v) \quad \text{for } v \in S, P \in \PPP_v.
\end{equation*}

To speed up the solving, we encode that for every pair of nodes, at
most one of the arcs between them can exist. We add the following
redundant clauses
\begin{equation*}\label{eq:speedup}
  \neg \arc uv \vee \neg \arc vu \quad \text{for } u,v \in S.
\end{equation*}

Now, we describe the additional clauses required to satisfy the fortified
constraints, and thus conditions~C3 and~C5.
For every virtual edge
$\{u, v\} \in E_\virt$, we introduce a forced arc depending on the elimination
ordering using the following pair of clauses
\begin{equation*}\label{eq:virtual-edge}
  \ords uv \rightarrow \arc uv \wedge  
  \ords vu \rightarrow \arc vu \quad \text{for } \{u,v\}\in E_\virt.
\end{equation*}
This takes care of the fortified treewidth constraints, satisfying C3 
and ensuring that the edge $\{u, v\} \subseteq \chi(s)$ for some $s \in V(S)$.
Finally, we add the clauses that encode the forced arcs $E_\virt^\rightarrow$.
For each $v \in S$, $P \in \PPP_v$, and $(u,v) \in \Avirt vP$, we add the clause
\begin{equation*}\label{eq:virtual-arc}
  \ps vP \rightarrow \acycs uv, 
\end{equation*}
which forces the virtual arc $(u,v)$ if $P$ is the parent set of~$v$ in $D^\new$,
thereby handling the fortified acyclicity constraints and ensuring that C5 is
satisfied.
 
This concludes the definition of the MaxSAT instance, to which we will
refer as $\Phi_{D,f}(S)$. We refer to the weight of a satisfying
assignment $\tau$ of~$\Phi_{D,f}(S)$ as the sum of the weights of all
the soft clauses satisfied by~$\tau$.  Let
$\alpha(S):= \sum_{v \in S}f_\emptyset(v)$.
To each satisfying assignment $\tau$ of~$\Phi_{D,f}(S)$ we can
associate for each~$v\in V$ the corresponding parent set, which in
turn determines a directed graph~$D^\new$.  Due to
Theorem~\ref{the:slim}, the treewidth of~$M(D^\new)$ is bounded by
$W$, and $D^\new$  is acyclic. \longversion{\pagebreak[1]} By construction of
$\Phi_{D,f}(S)$, the weight of~$\tau$ equals
$\sum_{v\in S} f'_P(v)= f(D_S^\new) - \alpha(S)$. Conversely, if we
pick new parent sets for the vertices in $S$ such that all the
conditions C1--C5 are satisfied, then by construction of
$\Phi_{D,f}(S)$, the corresponding truth assignment $\tau$ satisfies
$\Phi_{D,f}(S)$, and its weight is
$\sum_{v\in S} f'_P(v)= f(D_S^\new) - \alpha(S)$. 
In particular, let $K_0$ be the weight of the truth assignment which
corresponds to the parent sets of~$S$ as defined by the input DAG $D$.
We summarize these
observations in the following theorem.

\begin{theorem}
$\Phi_{D,f}(S)$ has a solution of weight $K$ if and only if
there are new parent sets for the vertices in~$S$ giving rise to a DAG
$D^\new$ with $f(D^\new)-f(D)=  K-K_0$.
 \end{theorem}

\section{Experimental evaluation}
\newcommand{\fnurl}[1]{\url{#1}}
\newif\ifpadhead
\padheadtrue

\longversion{In this section, we describe  the experiments conducted to
analyze the performance of the local improvement algorithm.
}
The current
state-of-the-art heuristic algorithms for solving the treewidth-bounded 
BN structure learning problem
are the \kmax\ algorithm by~\citet{ScanagattaCZYK18} and
the \hcb\ algorithms by~\citet{BenjumedaCamposLarranaga19}
(available as two variants--the default variant~\hc\ and the poly-time
variant~\hcp),
therefore, we analyze the benefit of applying \bnslim\ on top of these algorithms. 
It is worth noting that both \kmax\ and \bnslim\ are anytime algorithms,
i.e., they run indefinitely long and can be halted at any instant to output
the best solution found so far;
\hcb, on the other hand, as per the available implementation, is deterministic
and terminates when it fails to find any new improvements.
This distinction affects the nature of the experiments conducted to draw a 
comparison between the different algorithms.
However, for the most part, we closely 
follow the experimental setup (including data sets, timeouts,
comparison metrics) used by~\citet{ScanagattaCZYK18} to compare \kmax\ with previous
approaches.


Since \bnslim\ needs an initial heuristic solution, we enlist either 
\kmax, \hc, or \hcp\ for this purpose. We denote by \bnslim(X), 
the algorithm which applies \bnslim\ on an initial solution provided by X
where X~$\in\{\text{\kmax, \hc, \hcp}\}$.
We run all our experiments with treewidth bounds 2, 5, 8 for each data set
following~\citet{ScanagattaCZYK18}. All reported \bnslim\ results are
averages over three random seeds (see supp. material for details).

\subsection{Setup}\label{subsec:exp-setup}
We run all our experiments on a 4-core Intel Xeon E5540 2.53 GHz CPU, with each
process having access to 8GB RAM. 
We use UWrMaxSat as the MaxSAT-solver primarily due to its anytime nature
(available at the 2019 MaxSAT Evaluation 
webpage\footnote{\fnurl{https://maxsat-evaluations.github.io/2019/descriptions.html}}).
We tried other solvers but found that UWrMaxSat works best for our use case.
We use the 
BNGenerator package~\cite{bngenerator} in conjunction with the BBNConvertor 
tool~\cite{bnconvertor} to generate and reformat random Bayesian Networks.
We also use the implementation of the \kmax\ algorithm available as a
part of the BLIP package~\cite{blippackage}.
For the \hcb\ algorithms we use the software made 
available\footnote{\fnurl{https://github.com/marcobb8/et-learn}} 
by~\citet{BenjumedaCamposLarranaga19}.
We implement the local improvement algorithm in Python 3.6.9, using the
NetworkX 2.4 graph library~\cite{networkxlib}.
\shortversion{%
The source code is attached as supplementary material, and we 
intend to make it publicly available.
}
\longversion{%
The source code along with the experiment data is available publicly
at~\url{https://github.com/aditya95sriram/bn-slim}.
}

We first conducted a preliminary analysis on 20 data sets
to find out the best values for the \emph{budget} (maximum number of random variables
in a subinstance) and the \emph{timeout} (per MaxSAT call) of \bnslim.
We tested out budget values 7, 10, and 17, and timeout
values~1\text{s}, 2\text{s}, and 5\text{s}, and finally settled on 
a budget of 10 and a timeout of 2 seconds for our experiments.

\subsection{Data sets}\label{subsec:datasets}
We consider 99 data sets for our experiments. 
84 of these come from \emph{real-world} benchmarks. 
These are based on the benchmarks introduced 
by~\citet{LowdJesse10,VanDavis12,BekkerDCDV15,LarochelleBengioTurian10},
a subset of which has been used by~\citet{ScanagattaCZYK18}.
These benchmarks are publicly available%
\footnote{\fnurl{https://github.com/arranger1044/DEBD}}
in the form of pre-partitioned data sets. 
There are three data sets corresponding to each of the 28 benchmarks 
(see Table~\ref{tab:real-datasets}).

The remaining 15 data sets are classified as \emph{synthetic} as they are obtained by 
drawing 5000 samples from known BNs (see Table~\ref{tab:synth-datasets}).
Five of these BNs are commonly used in
the literature as benchmarks\footnote{\fnurl{https://www.bnlearn.com/bnrepository/}},
and we generated the remaining 10 BNs randomly using the BNGenerator tool
with more random variables than the previously mentioned data sets.
Overall, the collection of data sets
provides a wide variety of the data's nature and the different
parameters.

Both \kmax\ and \bnslim\ take a score function cache as input,
while \hcb\ requires the samples themselves
and computes the required scores on-the-fly.
We thus compute the score function cache using the scoring module provided as 
a part of \hcb's source code. More specifically, we first obtain the parent set
tuples using independence selection (available in the BLIP package), and then
we recompute the scores for these tuples using \hcb's scoring module.
This cache is used as input to both \bnslim\ and \kmax.
This provides a level playing field and improves comparability
between the different algorithms.

While computing these score function caches, the scoring function module
was unable to process two data sets and hence we discarded these two data sets. The final list of data sets is
shown in Tables~\ref{tab:real-datasets} and~\ref{tab:synth-datasets}.
Further, \kmax\ crashes for 3 data sets and hence we disregard these
for any experiments involving \kmax\ or \kslim.
\begin{table}[htbp]
	\centering\small
\setlength{\tabcolsep}{5.1pt}
\begin{tabular}{@{}lclclc@{}}
\toprule
Name   & $n$   & Name & $n$   & Name & $n$    \\
\cmidrule(r){1-2} \cmidrule(lr){3-4} \cmidrule(l){5-6}
NLTCS     & \phantom{1}16  & Connect 4   & 126 & EachMovie     & \phantom{1}500  \\
MSNBC     & \phantom{1}17  & OCR Letters & 128 & WebKB         & \phantom{1}839  \\
KDDCup2k  & \phantom{1}65  & RCV-1       & 150 & Reuters-52    & \phantom{1}889  \\
Plants    & \phantom{1}69  & Retail      & 135 & 20 NewsGroup  & \phantom{1}910  \\
Audio     & 100 & Pumsb-star  & 163 & Movie reviews & 1001 \\
Jester    & 100 & DNA         & 180 & BBC           & 1058 \\
Netflix   & 100 & Kosarek     & 190 & Voting        & 1359 \\
Accidents & 111 & MSWeb       & 294 & Ad            & 1556 \\
Mushrooms & 112 & NIPS        & 500 &               &      \\
Adult     & 123 & Book        & 500 &               &      \\ 
\bottomrule
\end{tabular}

    \vspace{-5pt}
	\caption{Real data sets ($n$ is the number of random
          variables, the number of samples ranges from 100 to 291326)}
	\label{tab:real-datasets}
	\longversion{
	\end{table}
	\begin{table}[htbp]
		\centering\small
	}
    \shortversion{\vspace{3pt}}
\begin{tabular}{@{}lclclc@{}}
\toprule
Name     & $n$    & Name & $n$    & Name & $n$ \\ 
\cmidrule(r){1-2} \cmidrule(lr){3-4} \cmidrule(l){5-6}
andes    & 223  & r0   & 2000 & r5   & 4000 \\
diabetes & 413  & r1   & 2000 & r6   & 4000 \\
pigs     & 441  & r2   & 2000 & r7   & 4000 \\
link     & 724  & r3   & 2000 \\
munin    & 1041 & r4   & 2000 \\ \bottomrule
\end{tabular}
    \vspace{-5pt}
	\caption{Synthetic data sets
		($n$ \longversion{denotes}\shortversion{is} the number of random variables,
		5000 samples from each network)}
	\label{tab:synth-datasets}
    \vspace{-15pt}
\end{table}

\subsection{Evaluation metric}

For evaluating our algorithm's performance, we use the same metric
as~\citeauthor{ScanagattaCZYK18}, i.e., $\Delta$BIC, which is the difference between
the BIC scores of two solutions. Given a DAG $D$, the BIC score  
approximates  the logarithm of the marginal likelihood of~$D$. Thus, given
two DAGs $D_1$ and $D_2$, the difference in their BIC scores approximates the 
ratio of their respective marginal likelihoods which is the Bayes Factor~\cite{Raftery95}.
A positive $\Delta$BIC score signifies positive evidence towards $D_1$ and a
negative $\Delta$BIC score signifies positive evidence towards $D_2$. The 
$\Delta$BIC values can be mapped to a scale of qualitative categories~\cite{Raftery95} 
\shortversion{%
as follows: \\[5pt]
	{\small\input{tables/deltabic-categories.tex}\par}%
}
\longversion{%
as shown in Table~\ref{tab:deltabic}.

\begin{table}[htbp]
	\centering
\newcommand\crule[1]{\textcolor{#1}{\rule[-2pt]{10pt}{10pt}}}
\definecolor{cmap1}{HTML}{43AAC7}
\definecolor{cmap2}{HTML}{7DC4D8}
\definecolor{cmap3}{HTML}{B8DFEA}
\definecolor{cmap4}{HTML}{F2F2F2}
\definecolor{cmap5}{HTML}{FAC6C8}
\definecolor{cmap6}{HTML}{F79EA2}
\definecolor{cmap7}{HTML}{F4787E}
\setlength{\tabcolsep}{5.8pt}
\begin{tabular}{@{}lclc@{}}
\toprule
\crule{white} Category            & $\Delta$BIC      & \crule{white} Category            & $\Delta$BIC     \\
\cmidrule(r){1-2} \cmidrule(l){3-4}
\crule{cmap7} extremely negative  & $(-\infty, -10)$ & \crule{cmap1} extremely positive  & $(10, \infty)$  \\
\crule{cmap6} strongly negative   & $(-10,-6)$       & \crule{cmap2} strongly positive   & $(6, 10)$       \\
\crule{cmap5}  negative           & $(-6,-2)$        & \crule{cmap3} positive            & $(2, 6)$        \\
\bottomrule
\end{tabular}

	\caption{$\Delta$BIC category scale}
	\label{tab:deltabic}
\end{table}
}

\subsection{Experimental results}


The primary focus of our experimentation is to analyze the benefit
gained by applying \bnslim\ on top of other heuristics and not to
compare between the different heuristics. To this end,
we run \bnslim\ for 60 minutes on top of the initial solution
provided by \kmax, \hc, and \hcp\ and measure the time required for
\bnslim\ to obtain a solution that counts as extremely positive
evidence with respect to the initial solution.  The initial
solution by \kmax\ is the solution captured at the 30-minute mark,
whereas the initial solution by \hcb\ is the final solution obtained
upon termination.
The maximum time required for computing the
initial solution on any individual instance, by both \hc\ and \hcp,
is around 3.5 hours.  For comparison, we let \kmax\ continue
running for 60 more minutes after it has produced the
initial solution. 

\begin{figure*}[b]
	\centering
    \longversion{
        \includegraphics[width=\linewidth]{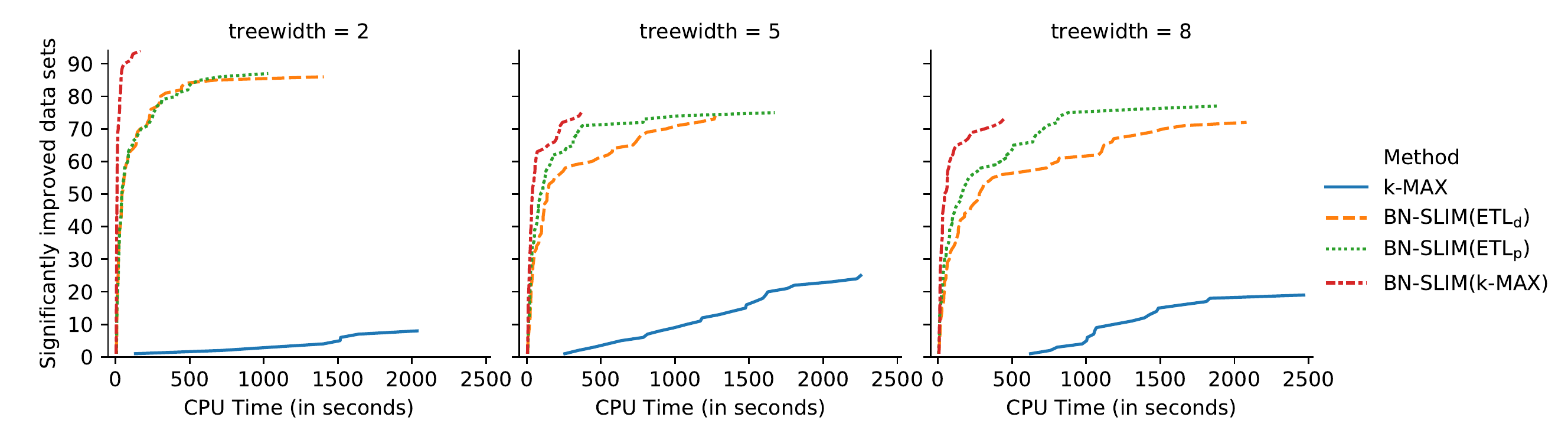}
        \vspace{-23pt}
    }
    \shortversion{
        \includegraphics[width=.9\linewidth]{figures/cdf-plot.pdf}
        \vspace{-12pt}
    }
	\caption{CDF plots showing the number of significantly improved
    data sets ($\Delta$BIC $\geq$ 10) across 94 data~sets}
	\label{fig:cactus}
    \vspace{-10pt}
\end{figure*}

Fig.~\ref{fig:cactus} shows the results of this analysis. 
We consider a data set to be \emph{significantly improved} if
\bnslim\ is able to improve by at least 10 BIC points over the initial
heuristic solution.
We observe that \bnslim\ improves over \kmax\ much more efficiently as
over \hcb. Giving \kmax\ more time for computing the
  initial solution increases this discrepancy even further, as the
  improvement rate of \kmax\ rapidly slows down after 30 minutes.
Averaging over all the heuristics, \bnslim\ can  produce a
solution with extremely positive evidence for 95\%, 79\%, and 78\%
of instances for treewidth bounds 2, 5, and 8, respectively.

Fig.~\ref{fig:slim-hc} shows the $\Delta$BIC values from comparing the
\hcbslim\ solution after 30 minutes to the corresponding initial
solution by \hcb.  We can see that \hcbslim\ can secure
extremely positive evidence for a significant number of data sets
across all tested treewidth bounds, with a smaller
treewidth being more favorable.

Due to the anytime nature of \kmax, we can compare it against
\kslim\ in a ``race.''  We run both simultaneously for one hour, where
out of the time allotted to \kslim, 30 minutes are used to generate
the initial solution,
and the remaining 30 minutes are used 
to improve this initial solution.  Fig.~\ref{fig:slim-kmax} shows the
$\Delta$BIC values of comparing \kmax\ and \kslim\ at the one hour
mark.  Similar to \hcbslim\ we observe that \kslim\ outperforms \kmax\
on a significant number of instances, and on all instances for treewidth 2.


The experimental evaluation demonstrates 
\bnslim\ approach's effectiveness
and the combined power as a heuristic method of \kslim\ and \hcbslim.



\begin{figure}[htb]
	\shortversion{%
        \vspace{-12pt}
		\centering
        \includegraphics[width=\linewidth]{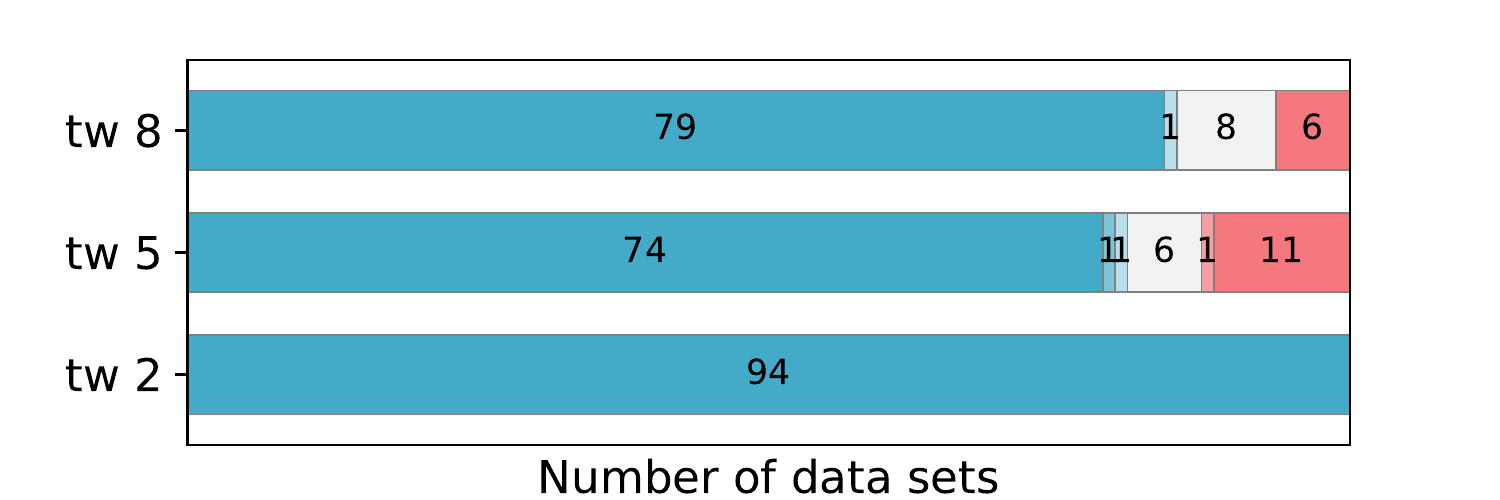}
        \vspace{-15pt}
	}
	\longversion{%
		\centering
		\includegraphics[width=.65\linewidth]{figures/slim-kmax-30m.pdf}
	}
	\caption{Comparison between \kslim\ and \kmax\ over 94 data sets
    \note{R2: avoid using color as only differentiator}}
	\label{fig:slim-kmax}
    \shortversion{\vspace{-15pt}}
\end{figure}

\begin{figure}[htb]
    \shortversion{\vspace{-10pt}}
	\centering
    \shortversion{\includegraphics[width=\linewidth]{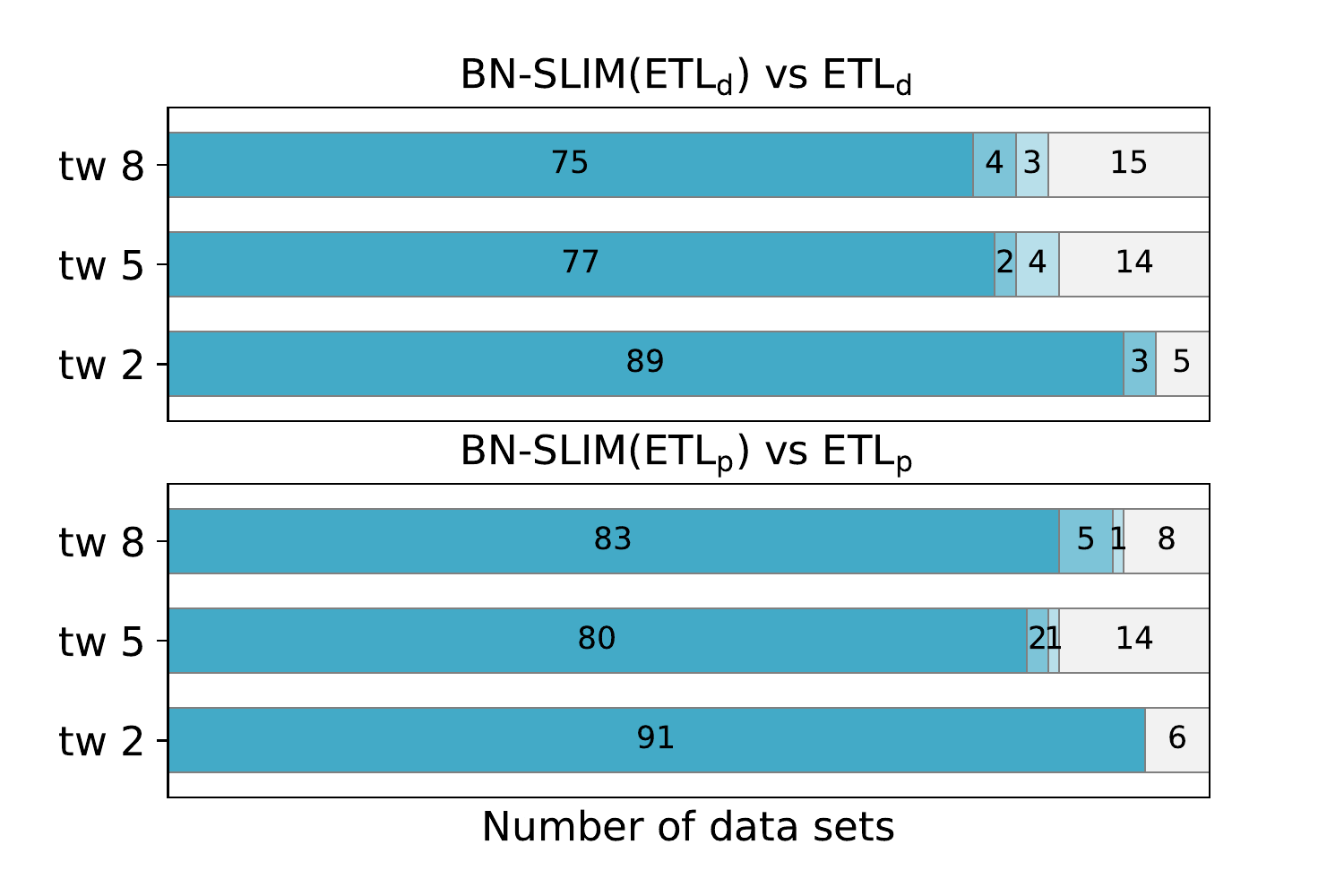}}
    \shortversion{\vspace{-23pt}}
    \longversion{\includegraphics[width=.65\linewidth]{figures/slim-hc-30m.pdf}}
	\caption{Comparison between \hcbslim\ and \hcb\ over 97 data sets}
	\label{fig:slim-hc}
    \shortversion{\vspace{-20pt}}
\end{figure}





\section{Conclusion}
With \bnslim, we have presented a novel method for improving the
outcome of treewidth-bounded BN structure learning heuristics. We have
demonstrated its robustness and performance by applying \bnslim\ to
the solution provided by the state-of-the-art heuristics \kmax, \hc,
and \hcp. The approach of \bnslim\ is based on exact reasoning via
MaxSAT, which is fundamentally different from the mentioned
heuristics. Consequently, both approaches complement each other, and
their combination provides significantly better solutions than any of
the heuristics alone.  Simultaneoulsy, the combination still scales to
large instances with thousands of random variables, which are far out
of reach for exact methods alone. Thus, \bnslim\ combines the best of
both worlds.

The highly encouraging experimental outcome suggests several avenues
for future work, which include the development of more sophisticated
subinstance selection schemes, the inclusion of variable fidelity
sampling (crude for the global solver, fine-grained for the local
solver), as well as more complex collaboration protocols between local
and global solver in a distributed setting.


  


\longversion{
\setlength{\bibsep}{0.5\baselineskip}
\bibliographystyle{plainnat}
}
\shortversion{
\section*{Acknowledgements}
The authors acknowledge the support by the FWF (projects P32441 and W1255) and by the WWTF (project ICT19-065).
}

\bibliography{literature}

\end{document}